%% file: confstab.tex
\documentclass{llncs}
\usepackage{graphicx} 
\usepackage{subfigure} 
\usepackage{times}

\usepackage{algorithm}
\usepackage{algorithmic}

\input{header}
\newcommand{\confusion}{{\cal C}}
\newcommand{\algo}{{\cal A}}
\newcommand{\dilation}{\mathfrak{D}}
\newcommand{\bfloss}{\boldsymbol{\loss}}
\newcommand{\bfone}{{\bf 1}}

\begin{document} 

\title{Confusion Matrix Stability Bounds for Multiclass Classification}
\author{Pierre Machart \and Liva Ralaivola}
\institute{QARMA, LIF UMR CNRS 7279\\
 Aix-Marseille Universit\'e\\
  39, rue F-Joliot Curie, F-13013 Marseille, France\\
\email{\{pierre.machart,liva.ralaivola\}@lif.univ-mrs.fr}}

\maketitle

\begin{abstract} 
We provide new theoretical results on the
generalization properties of learning algorithms for multiclass
classification problems. The originality of
our work is that we propose to use the {\em confusion matrix}  of a
classifier as a measure of its quality; our contribution is in the
line of work which attempts to set up and study the statistical
properties of new evaluation measures such as, e.g. ROC curves.
In the confusion-based learning framework we propose, we claim that a targetted
objective is to minimize the size of the confusion matrix $\confusion$,
measured through its {\em operator norm}
$\|\confusion\|$. We derive generalization bounds on the (size of the)
confusion matrix in an extended framework of uniform stability,
adapted to the case of matrix valued loss. Pivotal to our study is a
very recent matrix concentration inequality that generalizes
McDiarmid's inequality. As an illustration of the relevance of our
theoretical results, we show how two SVM learning procedures can be
proved to be confusion-friendly. To the best of our knowledge, the
present paper is the first that focuses on the confusion matrix from a
theoretical point of view.
\end{abstract} 

\input{intro}
\input{setting}
\input{stability}
\input{analysis}
\input{conclusion}

\bibliographystyle{splncs}
\bibliography{confstab}

\end{document}

%% file: header.tex
\usepackage{amsmath,amsfonts,amssymb}
\usepackage{array}
\usepackage{xspace}
\usepackage{graphicx}
\usepackage{color}
\usepackage{algorithmic, algorithm,verbatim}
\usepackage{enumerate}

\DeclareMathOperator*{\argmin}{argmin}

\newcommand{\indicator}[1]{\mathbb{I}_{\left\{#1\right\}}}

\def\inputset{{\bf Z}}

\def\distribution{D}
\def\proba{\mathbb{P}}
\def\inputspace{\mathcal{X}}

\def\targetspace{\mathcal{Y}}
\def\productspace{\mathcal{Z}}
\def\setS{\mathcal{S}}
\def\rkhs{\mathcal{H}}
\def\proba{\mathbb{P}}
\def\expectation{\mathbb{E}}
\def\realset{\mathbb{R}}

\DeclareMathOperator*{\union}{\cup}

\def\family{\mathcal{H}}
\def\loss{\ell}
\def\Loss{L}
\def\risk{R}

\def\identity{I}

\def\hypotheses{\mathcal{H}}

\def\mmin{m^*}
\def\betamin{\beta^*}

\def\Lq{\mathcal{L}_q}
\def\Lqemp{\hat{\mathcal{L}}_q}
\def\Lqi{\mathcal{L}_q^i}
\def\Lqiemp{\hat{\mathcal{L}}_q^i}
\def\Lqnoi{\mathcal{L}_q^{\backslash i}}
\def\Lqnoiemp{\hat{\mathcal{L}}_q^{\backslash i}}

\def\bfx{\boldsymbol{x}}

\def\bfz{\boldsymbol{z}}

\def\bfZ{\boldsymbol{Z}}
\def\bfz{\boldsymbol{z}}

\def\bfs{\boldsymbol{s}}
\def\bfX{\boldsymbol{X}}
\def\bfY{\boldsymbol{Y}}

\def\bfv{\boldsymbol{v}}
\def\bfy{\boldsymbol{y}}

\def\bfZ{\boldsymbol{Z}}

\def\bfpi{\boldsymbol{\pi}}

\def\bfdelta{\boldsymbol{\delta}}
\def\bflambda{\boldsymbol{\lambda}}

\usepackage{exscale}
 

\makeatletter
\newenvironment{equationsize*}[1]{%
  \skip@=\baselineskip 
  #1%
  \baselineskip=\skip@ 
  \equation
}{\nonumber\endequation \ignorespacesafterend} 
\makeatother

 

\makeatletter
\newenvironment{alignsize*}[1]{%
  \skip@=\baselineskip 
  #1%
  \baselineskip=\skip@ 
  \start@align\@ne\st@rredtrue\m@ne
}{\endalign\ignorespacesafterend} 
\makeatother



%% file: intro.tex
\section{Introduction}
\label{sec:introduction}
Multiclass classification is an important problem of machine
learning. The issue of having at hand statistically relevant procedures
to learn reliable predictors is of particular interest, given
the need of such predictors in information retrieval, web
mining, bioinformatics or neuroscience (one may for example think of document
categorization, gene classification, fMRI image classification).

Yet, the literature on multiclass learning is not as voluminous as
that of binary classification, while this multiclass prediction raises questions
from the algorithmic, theoretical and practical points of view. One of
the prominent questions is that of the measure to use in order to
assess the quality of a multiclass predictor. Here, we develop our
results with the idea that the {\em confusion matrix} is a performance
measure that deserves to be studied as it provides a finer information
on the properties of a classifier than the mere misclassification
rate. We do want to emphasize  that we provide theoretical results
on  the confusion matrix itself and that misclassification rate {\em
  is not} our primary concern ---as we shall see, though, getting
bounds on the confusion matrix entails, as a byproduct, bounds on the
misclassification rate.

Building on matrix-based
concentration inequalities~\cite{recht11simpler,tropp11a,gosh11who,rudelson07sampling,chaudhuri09multiview}, also
referred to as noncommutative concentration inequalities, we
establish a stability  framework for confusion-based learning
algorithm. In particular, we prove a generalization bound for
{\em confusion stable} learning algorithms and show that there exist
such algorithms in the literature.  In a sense, our framework and our results extend those of
\cite{bousquet02}, which are designed for scalar  loss functions. To
the best of our knowledge, this is the first work that establishes
generalization bounds based on confusion matrices.

The paper is organized as follows. Section~\ref{sec:setting} describes
the setting we are interested in and motivates the use of the
confusion matrix as a performance measure. Section~\ref{sec:results}
introduces the new notion of {\em stability} that will prove essential
to our study; the main theorem of this paper, together with its proof,
are provided. Section~\ref{sec:analysis} is devoted to
the analysis of two {\em SVM} procedures in the light of our new
framework. A discussion on the merits and possible extensions of our
approach concludes the paper (Section~\ref{sec:conclusion}).


%% file: setting.tex
\section{Confusion Loss}
\label{sec:setting}
\subsection{Notation}

As said earlier, we focus on the problem of multiclass classification.
The input space
is denoted by $\inputspace$ and the target space is
$$\targetspace=\{1,\ldots,Q\}.$$ The training sequence
$$\inputset=\{Z_i=(X_i,Y_i)\}_{i=1}^m$$ is made of $m$ identically and
independently random pairs $Z_i=(X_i,Y_i)$ distributed according to some unknown
(but fixed) distribution $\distribution$ over
$\productspace=\inputspace\times\targetspace$. The sequence of input data will be
referred to as ${\bfX}=\{X_i\}_{i=1}^m$ and the sequence of corresponding labels
${\bfY}=\{Y_i\}_{i=1}^m$, we may write $\inputset=\{\bfX,\bfY\}$.
The realization of $Z_i=(X_i,Y_i)$ is $z_i=(x_i,y_i)$  and 
$\bfz$, $\bfx$ and $\bfy$ refer to the realizations of the
corresponding sequences of random variables.
For a sequence $\bfy=\{y_1,\cdots,y_m\}$ of $m$ labels, $m_q(\bfy)$, or
simply $m_q$ when clear from context, denotes the number of 
labels from $\bfy$ that are equal to $q$; $\bfs(\bfy)$ it the binary
sequence $\{s_1(\bfy),\ldots,s_Q(\bfy)\}$ of size $Q$ such that $s_q(\bfy)=1$ 
if $q\in\bfy$ and $s_q(\bfy)=0$ otherwise.
 
 We will use
$D_{X|y}$  for the conditional distribution of $X$ given that $Y=y$;
therefore, for a given
sequence $\bfy=\{y_1,\ldots,y_m\}\in\targetspace^m$,
$D_{\bfX|\bfy}=\otimes_{i=1}^m\distribution_{X|y_i}$ is the distribution of
the random sample $\bfX=\{X_1,\ldots,X_m\}$ over $\inputspace^m$ such that
$X_i$ is distributed according to
$\distribution_{X|y_i}$; for $q\in\targetspace$, and $\bfX$
distributed according to $D_{\bfX|\bfy}$, $\bfX_q=\{X_{i_1},\ldots,X_{i_{m_q}}\}$ denotes the
random sequence of variables such that $X_{i_k}$ is distributed
according to $\distribution_{X|q}$.
$\expectation[\cdot]$ and $\expectation_{X|y}[\cdot]$ denote the
expectations with respect to $\distribution$ and
$\distribution_{X|y}$, respectively.

For a training sequence $\inputset$, $\inputset^i$ denotes the
sequence $$\inputset^i=\{Z_1,\ldots Z_{i-1},Z_i',Z_{i+1},\ldots,Z_m\}$$ where $Z_i'$
is distributed as $Z_i$; $\inputset^{\backslash i}$ is the sequence
$$\inputset^{\backslash i}=\{Z_1,\ldots
Z_{i-1},Z_{i+1},\ldots,Z_m\}.$$ These definitions
directly carry over when conditioned on a sequence of labels $\bfy$ (with,
henceforth, $y_i'=y_i$).

We will consider a family $\family$ of predictors such that 
$$\family\subseteq\left\{h:h(x)\in\realset^{Q},\;\forall x\in\inputspace\right\}.$$
For $h\in\family$, $h_q\in\realset^{\inputspace}$ denotes its $q$th
coordinate. Also, $$\bfloss=(\loss_q)_{1\leq q\leq Q}$$ is a set of loss
functions such that: $$\loss_q:\family\times\inputspace\times\targetspace\rightarrow\realset_+.$$
Finally, for a given algorithm $\algo:\union_{m=1}^{\infty}\productspace^m\rightarrow\hypotheses$, $\algo_\inputset$ will denote the
hypothesis learned by $\algo$ when trained on $\inputset$.

\subsection{Confusion Matrix versus Misclassification Rate}
We here provide a discussion as to why minding the {\em confusion
  matrix} or {\em confusion loss} (terms that we will use
interchangeably) is
crucial in multiclass classification. We also introduce the reason why
we may see the confusion matrix as an operator and, therefore,
motivate the recourse to the {\em operator norm} to measure the `size' of
the confusion matrix.

In many situations, e.g. class-imbalanced datasets, it is important
not to measure the quality of a predictor $h$ on its classification error
$\proba_{XY}( h(X)\neq Y)$ only, as this may lead to erroneous
conclusions regarding the quality of $h$. Indeed, if, for instance,
some class $q$ is predominantly present in the data at hand, say
$\proba(Y=q)=1-\varepsilon$, for some small $\varepsilon>0$,  then the predictor $ h_{\text{maj}}$ that
always outputs $ h_{\text{maj}}(x)=q$ regardless of $x$ has a
classification error lower than $\varepsilon$. Yet, it might be important
not to classify an instance of some class $p$ in class $q$:
take the example of classifying mushrooms according to
the categories $\{\text{\tt hallucinogen},\text{\tt poisonous},
\text{\tt innocuous}\}$, it might not be benign to predict {\tt innocuous}
(the majority class) instead of {\tt hallucinogen} or {\tt poisonous}. The framework we consider allows us,
among other things, to be
immune to situations where class-imbalance may occur.

We do claim that a more relevant object to consider is the
{\em confusion matrix} which, given a binary sequence
$\bfs=\{s_1\cdots s_Q\}\in\{0,1\}^Q$, is defined as
$$\confusion_{\bfs}(h):=\sum_{q:s_q=1}\expectation_{X|q}\Loss(h,X,q),$$
where, given an hypothesis $h\in\family$, $x\in\inputspace$, $y\in\targetspace$,
$L(h,x,y)=(l_{ij})_{1\leq i,j\leq Q}\in\realset^{Q\times Q}$ is the \emph{loss matrix} such that:
\begin{align*}
 l_{ij} := \left\{ \begin{array}{ll}
					 \loss_j(h,x,y) &\text{if } i=y
                                         \text{ and } i\neq j\\
                                         0 & \text{otherwise.}
                                        \end{array}\right.
\end{align*}
Note that this matrix has at most one nonzero row, namely its $i$th row.
 
For a sequence
$\bfy\in\targetspace^m$ of $m$ labels and a random sequence
$\bfX$ distributed according to $\distribution_{\bfX|\bfy}$, the
conditional empirical confusion matrix
$\widehat{\confusion}_{\bfy}(h,\bfX)$ is 
\begin{equation*}
\widehat{\confusion}_{\bfy}(h,\bfX):=\sum_{i=1}^m\frac{1}{m_{y_i}}\Loss(h,X_i,y_i)=\sum_{q\in\bfy}\frac{1}{m_q}\sum_{i:y_i=q}\Loss(h,X,q)=\sum_{q\in\bfy}\Loss_q(h,\bfX,\bfy),
\end{equation*}
where
\begin{equation*}
\Loss_q(h,\bfX,\bfy):=\frac{1}{m_q}\sum_{i:y_i=q}\Loss(h,X_i,q).
\end{equation*}
For a random sequence $\bfZ=\{\bfX,\bfY\}$ distributed according to
$\distribution^m$, the (unconditional) empirical confusion matrix is given by
 $$\expectation_{\bfX|\bfY}\widehat{\confusion}_{\bfY}(h,\bfX)=\confusion_{\bfs(\bfY)}(h),$$ 
which is a random variable, as it depends on the random sequence
$\bfY$. For exposition purposes it will often be more convenient to
consider a fixed sequence $\bfy$ of labels and state results on
$\widehat{\confusion}_{\bfy}(h,\bfX)$, noting that
$$\expectation_{\bfX|\bfy}\widehat{\confusion}_{\bfy}(h,\bfX)=\confusion_{\bfs(\bfy)}(h).$$

The slight differences between our definitions of (conditional)
confusion matrices and the usual definition of a confusion matrix 
is that the diagonal elements are all zero and that
they can accomodate any family of loss functions (and not just the
$0$-$1$ loss).

A natural objective that may be pursued in multiclass classification
is to learn a classifier $h$ with `small' confusion matrix, where
`small' might be defined with respect to (some) matrix norm of
$\confusion_{\bfs}( h)$. The norm that we retain is
the {\em operator
norm} that we denote $\|\cdot\|$ from now on: recall that, for a matrix $M$,
$\|M\|$ is computed as
$$\|M\|=\max_{\bfv\neq {\bf 0}}\frac{\|M\bfv\|_2}{\|\bfv\|_2},$$
where $\|\cdot\|_2$ is the Euclidean norm; $\|M\|$ is merely the
largest singular value of $M$ ---note that $\|M^{\top}\|=\|M\|.$

Not only is the operator norm a `natural' norm on matrices but an
important reason  for working with it is that
$\confusion_{\bfs}(h)$ is often precisely used as an {\em operator} acting
on the vector of prior distributions
$$\bfpi=[\proba(Y=1)\cdots\proba(Y=Q)]^{\top}.$$
Indeed, a quantity of interest is for instance the $\bfloss$-{\em risk}
$\risk_{\bfloss}( h)$ of $ h$, with
\begin{align*}
\risk_{\bfloss}( h)
&:=\expectation_{XY}\left\{\sum_{q=1}^Q\loss_{q}( h,X,Y)\right\}=\expectation_Y\left\{\sum_{q=1}^Q\expectation_{X|Y}\loss_{q}( h,X,Y)\right\}\\
&=\sum_{p,q=1}^Q\expectation_{X|p}\loss_{q}( h,X,p)\pi_p=\|\bfpi^{\top}\confusion_{\bfone}( h)\|_1.
\end{align*}
It is interesting to observe that, $\forall h,\forall\bfpi\in\Lambda:=\{\bflambda\in\realset^Q:\lambda_q\geq 0,\, \sum_q\lambda_q=1\}$:
\begin{align*}
0\leq\risk_{\bfloss}( h)&=\|\bfpi\confusion_{\bfone}(h)\|_1=\bfpi^{\top}\confusion_{\bfone}( h)
\bfone\\
&\leq\sqrt{Q}\left\|\bfpi^{\top}\confusion_{\bfone}( h)\right\|_2 =\sqrt{Q}\left\|\confusion_{\bfone}^{\top}( h) \bfpi \right\|_2\\
& \leq\sqrt{Q}\left\|\confusion_{\bfone}^{\top}( h) \right\|\| \bfpi\|_2 \leq \sqrt{Q}\left\|\confusion_{\bfone}^{\top}( h) \right\|=\sqrt{Q}\left\|\confusion_{\bfone}( h)\right\|,
\end{align*}
 where we have used Cauchy-Schwarz inequalty in the second line, the
definition of the operator norm on the third line and the fact that
$\|\bfpi\|_2\leq 1$ for any $\bfpi$ in $\Lambda$;
$\bfone$ is the $Q$-dimensional vector where each entry is $1$. 
Recollecting things, we just established the following proposition.
\begin{proposition}
\label{prop:errorandconfusion}
$\forall h\in\family,\;\risk_{\bfloss}(h)=\|\bfpi^{\top}\confusion_{\bfone}( h)\|_1\leq\sqrt{Q}\left\|\confusion_{\bfone}(h)\right\|.$
\end{proposition}
This precisely says that the operator norm of the confusion matrix (according
to our definition) provides a bound on the risk. As a consequence,
bounding $\|\confusion_{\bfone}(h)\|$ is a relevant way to bound the
risk in a way that is independent from the class priors (since the
$\confusion_{\bfone}(h)$ is independent form these prior
distributions as well). This is essential in class-imbalanced problems
and also critical if sampling (prior) distributions are different for
training and test data.

Again, we would like to insist on the fact that the confusion matrix {\em is
the subject of our study} for its ability to provide fine-grain
information on the prediction errors made by classifiers; as mentioned
in the introduction, there are
application domains where confusion
matrices {\em indeed are} the measure of performance that is looked at. If
needed, the norm of
the confusion matrix allows us to summarize the characteristics of the
classifiers in one scalar value (the larger, the worse), and it
provides, as a (beneficial) ``side effect'', a bound on 
$\risk_{\bfloss}(h)$.


%% file: stability.tex
\section{Deriving Stability Bounds on the Confusion Matrix}
\label{sec:results}
One of the most prominent issues in \emph{learning theory} is to estimate the real performance of a learning system.
The usual approach consists in studying how empirical measures converge to their expectation.
In the traditional settings, it often boils down to providing bounds describing how the empirical risk relates to the expected one.
In this work, we show that one can use similar techniques to provide bounds on (the operator norm of) the confusion loss.

\subsection{Stability}
Following the early work of \cite{vapnik82}, the risk has traditionally been estimated through its empirical measure and a measure of the complexity of the
hypothesis class such as the Vapnik-Chervonenkis dimension, the
fat-shattering dimension or the Rademacher complexity.
During the last decade, a new and successful approach based on
\emph{algorithmic stability} to provide some new bounds has emerged.
One of the highlights of this approach is the focus on properties of
the learning algorithm at hand, instead of the richness of hypothesis class.
In essence, algorithmic stability results aim at taking  advantage from the way a given
algorithm actually explores the hypothesis space, which may lead to
tight bounds. The main results of \cite{bousquet02} were obtained using the definition of \emph{uniform stability}.
\begin{definition}[Uniform stability \cite{bousquet02}]
 An algorithm $\algo$ has \emph{uniform stability} $\beta$ with respect to loss function $\ell$ if the following holds:
\begin{align*}
 \forall \inputset\in\productspace^m, \forall i \in \{1,\dots,m\}, \|\loss(\algo_\inputset,.) - \loss(\algo_{\inputset^{\backslash i}},.)\|_{\infty} \leq \beta.
\end{align*}
\end{definition}

In the present paper, we now focus on the generalization of
stability-based results to confusion loss.
We introduce the definition of \emph{confusion stability}.
\begin{definition}[Confusion stability]
\label{defconfstab}
 An algorithm $\algo$ is \emph{confusion stable} with respect to the
 set of loss functions  $\bfloss$ {\em if}
 there exists a constant $B>0$ such that
$\forall i \in \{1,\ldots,m\},\forall\bfz\in\productspace^m$, whenever
$m_q\geq 2, \forall q\in\targetspace,$
$$\sup_{x\in\inputspace}\left\|\Loss(\algo_{\bfz},x,y_i)-\Loss(\algo_{\bfz^{\backslash
      i}}, x, y_i)\right\|\leq \frac{B}{m_{y_i}}.$$

From here on, $q^*$, $m^*$ and $\beta^*$ will  stand for
$$q^*:=\argmin_{q}m_q,\; m^*:=m_{q^*}, \text{ and } \beta^*:=B/m^*.$$
\end{definition}

\subsection{Noncommutative McDiarmid's Bounded Difference Inequality}
Centaral to the resulst of \cite{bousquet02} is a variation of Azuma's concentration
inequality, due to \cite{mcdiarmid89}.
It describes how a scalar function of independent random variables (the elements of our training set) concentrates
around its mean, given how changing one of the random variables impacts the value of the function.

Recently there has been an extension of McDiarmid's inequality to the
matrix setting \cite{tropp11a} .
For the sake of self-containedness, we recall this noncommutative bound.
\begin{theorem}[Matrix bounded difference (\cite{tropp11a}, corollary 7.5)]
\label{McD}
Let $H$ be a function that maps $m$ variables
from some space $\productspace$ to
a self-adjoint matrix of dimension $2Q$. Consider a sequence $\{A_i\}$ of fixed self-adjoint matrices that satisfy
\begin{align}
\label{McDcond}
\left( H(z_1,\ldots,z_i,\ldots,z_m) -
  H(z_1,\ldots,z_i',\ldots,z_m) \right)^2 \preccurlyeq A_i^2, 
\end{align}
for $z_i,z_i'\in\productspace$ and for $i=1,\ldots,m$, where
$\preccurlyeq$ is the (partial) order on self-adjoint matrices.
Then,  if
$\inputset$ is a random sequence of independent variables over
$\productspace$:
\begin{align*}
\forall t\geq 0,\; \mathbb{P}\left\{\left\|H(\inputset)-\mathbb{E}_{\inputset}H(\inputset)\right\|\geq t\right\} \leq 2Q e^{-t^2/8\sigma^2},
\end{align*}
where $\sigma^2 := \|\sum_i A_i^2\|$.
\end{theorem}

The confusion matrices we deal with are not necessarily self-adjoint,
as is required by the theorem. To make use of the 
theorem, we rely on the dilation $\mathfrak{D}(A)$ of  $A$, with
\begin{align*}\mathfrak{D}(A):=
 \begin{pmatrix}
  0 & A\\
 A^{*} & 0
 \end{pmatrix},
\end{align*}
where $A^{*}$ is the adjoint of $A$ (note that $\mathfrak{D}(A)$ is
self-adjoint) and on the result (see~\cite{tropp11a})
$$\|\mathfrak{D}(A)\|=\|A\|.$$

\subsection{Stability Bound}

The following theorem is the main result of the paper. It says that
the empirical confusion is close to the expected confusion whenever
the learning algorithm at hand exhibits confusion-stability
properties. This is a new flavor of the results of \cite{bousquet02}
for the case of matrix-based loss.

\begin{theorem}[Confusion bound]
\label{th:confusionbound}
 Let $\algo$ be a learning algorithm. Assume that all the loss
 functions under consideration take values in the range $[0;M]$. Let
 $\bfy\in\targetspace^m$ be a fixed sequence of labels.

If $\algo$ is a confusion stable as defined in Definition~\ref{defconfstab},
then, $\forall m\geq1,\;\forall\delta \in (0,1)$, the following holds,
with prob. $1-\delta$ over the random draw of $\bfX\sim\distribution_{\bfX|\bfy}$,
\begin{align*}
\left\|\widehat{\confusion}_{\bfy}(\algo,\bfX)-\confusion_{\bfs(\bfy)}(\algo)\right\| \leq 2B \sum_{q} \frac{1}{m_q} + Q \sqrt{8 \ln\left(\frac{Q^2}{\delta}\right)} \left(4 \sqrt{\mmin} \betamin + M \sqrt{\frac{Q}{\mmin}}\right).
\end{align*}
As a consequence, with probability $1-\delta$ over the random draw of $\inputset\sim\distribution^m$,
\begin{align*}
&\left\|\widehat{\confusion}_{\bfY}(\algo,\bfX)-\confusion_{\bfs(\bfY)}(\algo)\right\| \leq 2B \sum_{q} \frac{1}{m_q}+ Q \sqrt{8 \ln\left(\frac{Q^2}{\delta}\right)} \left(4 \sqrt{\mmin} \betamin + M \sqrt{\frac{Q}{\mmin}}\right).
\end{align*}
\end{theorem}
\begin{proof}[Sketch] The complete proof can be found in
  the next subsection. We here provide the skeleton of the proof.
  We proceed in 3 steps to get the first bound. 

\begin{enumerate}
\item {\bf Triangle inequality.} To
  start with, we know by the triangle inequality
\begin{align}
\|\widehat{\confusion}_{\bfy}(\algo,\bfX)-\confusion_{\bfs(\bfy)}(\algo)\|\notag &=\left\|\sum_{q\in\bfy}(\Loss_q(\algo_{\bfZ},\bfZ)-\expectation_{\bfX}\Loss_q(\algo_{\bfZ},\bfZ))\right\|\notag\\
&\leq\sum_{q\in\bfy}\left\|\Loss_q(\algo_{\bfZ},\bfZ)-\expectation_{\bfX}\Loss_q(\algo_{\bfZ},\bfZ)\right\|.\label{eq:triangleineq}
\end{align}
Using uniform stability arguments, we bound each summand
with probability $1-\delta/Q$.
\item {\bf Union Bound.} Then, using the union bound we get a bound on
$\|\widehat{\confusion}(\algo,\bfX)-\confusion_{\bfs(\bfy)}(\algo)\|$
that holds with probability at least $1-\delta$.
\item {\bf Wrap up.} Finally, recoursing to a simple argument, we express the obtained
bound solely with respect to $\mmin$.
\end{enumerate}
Among the three steps, the first one is the more involved and much
part of the proof is devoted to address it.

To get the bound with the unconditional confusion matrix
$\confusion_{\bfs(\bfY)}(\algo)$ it suffices to observe that for any
event ${\mathcal{E}(\bfX,\bfY)}$ that depends on $\bfX$ and $\bfY$,
such that for all sequences $\bfy$,
$\proba_{\bfX|\bfy}\{\mathcal{E}(\bfX,\bfy)\}\leq \delta$, the
following holds:
\begin{align*}
\proba_{XY}(\mathcal{E}(\bfX,\bfY))&=\expectation_{XY}\left\{\indicator{\mathcal{E}(\bfX,\bfY)}\right\}
=\expectation_{\bfY}\left\{\expectation_{\bfX|\bfY}\indicator{\mathcal{E}(\bfX,\bfY)}\right\}\\
&=\sum_{\bfy}\expectation_{\bfX|\bfY}\indicator{\mathcal{E}(\bfX,\bfY)}\proba_{\bfY}(\bfY=\bfy)
=\sum_{\bfy}\proba_{\bfX|\bfy}\{\mathcal{E}(\bfX,\bfy)\}\proba_{\bfY}(\bfY=\bfy)\\
&\leq \sum_{\bfy}\delta \proba_{\bfY}(\bfY=\bfy)=\delta,
\end{align*}
which gives the desired result.
\end{proof}\qed

\begin{remark}
If needed, it is straightforward to  bound 
$\|\confusion_{\bfs(\bfy)}(\algo)\|$ and
$\|\confusion_{\bfs(\bfY)}(\algo)\|$ by using the triangle inequality
$|\|A\|-\|B\||\leq\|A-B\|$ on the stated bounds. 
\end{remark}

\begin{remark}
A few comments may help understand the meaning of our main theorem.
First,  it is expected to get a bound expressed in terms of
$1/\sqrt{\mmin}$, since a) $1/\sqrt{m}$ is a typical rate encountered
in bounds based on $m$ data and b) the bound cannot be better than a
bound devoted to the least informed class (that would be in
$1/\sqrt{\mmin}$) ---resampling procedures may be a strategy to
consider to overcome this limit.
Second, this theorem says that it is a relevant idea to try and
minimize the empirical confusion matrix of a multiclass predictor
provided the algorithm used is stable ---as will be the case of the
algorithms analyzed in the following section. Designing algorithm that
minimize the norm of the confusion matrix is therefore an enticing challenge.
Finally, when $Q=2$, that is we are in a binary classification
framework, Theorem~\ref{th:confusionbound} gives a bound on the
maximum of the false-positive rate and the false-negative rate, since
this the operator norm of the confusion matrix precisely corresponds
to this maximum value.

\end{remark}
\subsection{Proof of Theorem~\ref{th:confusionbound}}
To ease the readability, we introduce  additional notation:
\begin{align*}
\Lq&:=\expectation_{X|q}{\Loss}(\algo_{\inputset},X,q), \quad\Lqemp:={\Loss_q}(\algo_{\inputset},\bfX,\bfy),\\
\Lqi&:=\expectation_{X|q}{\Loss}(\algo_{\inputset^i},X,q),\quad \Lqiemp:={\Loss_q}(\algo_{\inputset^i},\bfX^i,\bfy^i),\\
\Lqnoi&:=\expectation_{X|q}{\Loss}(\algo_{\inputset^{\backslash i}},X,q),\quad\Lqnoiemp:={\Loss_q}(\algo_{\inputset^{\backslash i}},\bfX^{\backslash i},\bfy^{\backslash i}).
\end{align*}
%
After using the triangle inequality in~\eqref{eq:triangleineq}, we need to provide a bound on each summand.
To get the result, we will, for each $q$, fix the $X_k$ such that $y_k \neq q$ 
and work with functions of $m_q$ variables.
Then, we will apply Theorem \ref{McD} for each
$$H_q(\bfX_q,\bfy_q):= \mathfrak{D}(\Lq) - \mathfrak{D}(\Lqemp).$$
To do so, we prove the following lemma
\begin{lemma}
$\forall q,\forall i, y_i=q$
$$(H_q(\bfZ_q)-H_q(\bfZ_q^i))^2\preccurlyeq
\left(\frac{4 B}{m_q}+\frac{\sqrt{Q}M}{m_q}\right)^2\identity.$$
\end{lemma}
\begin{proof}This is a proof that works in 2 steps.

Note that
\begin{align*}
\|H_q(\bfX_q,\bfy_q)-H_q(\bfX_q^i,\bfy_q^i)\|&=\|\mathfrak{D}(\Lq) - \mathfrak{D}(\Lqemp) - \mathfrak{D}(\Lqi) + \mathfrak{D}(\Lqiemp)\|\\
 &= \|\Lq - \Lqemp - \Lqi + \Lqiemp\|
\leq \|\Lq-\Lqi\| + \|\Lqemp - \Lqiemp\|.
\end{align*}

\paragraph{Step 1: bounding $ \|\Lq-\Lqi\|$.} 
We can trivially write:
\begin{align*}
 \|\Lq-\Lqi\| \leq \|\Lq-\Lqnoi\| + \|\Lqi-\Lqnoi\|
\end{align*}
Taking advantage of the stability of $\algo$:
\begin{align*}
 \|\Lq-\Lqnoi\| &= \left\|\expectation_{X|q}\left[{\Loss}(\algo_{\bfZ},X,q) - {\Loss}(\algo_{\bfZ^{\backslash i}},X,q)\right]\right\|\\
&\leq \expectation_{X|q}\left\|{\Loss}(\algo_{\bfZ},X,q) - {\Loss}(\algo_{\bfZ^{\backslash i}},X,q)\right\|\\
&\leq \frac{B}{m_q},
\end{align*}
and  the same holds for $\|\Lqi-\Lqnoi\|$, i.e.
$\|\Lqi-\Lqnoi\|\leq B/m_q$. Thus, we have:
\begin{align}
 \|\Lq-\Lqi\| \leq \frac{2B}{m_q}. \label{eq:lqlqi}
\end{align}

\paragraph{Step 2: bounding $\|\Lqemp-\Lqiemp\|$.} This is a little
trickier than the first step.
\begin{align*}
 \|\Lqemp-\Lqiemp\| &= \big\|{\Loss_q}(\algo_{\inputset},\bfZ)
 - {\Loss_q}(\algo_{\inputset^i},\bfZ^{i})\big\|\\
&\quad= \frac{1}{m_q}\Big\|\sum_{k:k\neq i,y_k=q}\Big(\Loss(\algo_{\inputset},X_k,q)
 - \Loss(\algo_{\inputset^i},X_k,q)\Big)\\
 &\qquad + \Loss(\algo_{\inputset},X_i,q) - \Loss(\algo_{\inputset^i},X'_i,q)\Big\|\\
&\quad\leq \frac{1}{m_q}\Big\|\sum_{k:k\neq i,y_k=q}\Big(\Loss(\algo_{\inputset^i},X_k,q)
 - \Loss(\algo_{\inputset^i},X_k,q)\Big)\Big\|\\
 &\qquad + \frac{1}{m_q}\Big\|\Loss(\algo_{\inputset},X_i,q) - \Loss(\algo_{\inputset^i},X'_i,q)\Big\|\\
\end{align*}
Using the stability argument as before, we have:
\begin{align*}
 \Big\|\sum_{k:k\neq i,y_k=q}&\Big(\Loss(\algo_{\inputset},X_k,q)
 - \Loss(\algo_{\inputset^i},X_k,q)\Big)\Big\|\\
&\leq \sum_{k:k\neq i,y_k=q}\left\|\Loss(\algo_{\inputset},X_k,q)
 - \Loss(\algo_{\inputset^i},X_k,q)\right\| \leq \sum_{k:k\neq i,y_k=q} 2 \frac{B}{m_q} \leq 2B.
\end{align*}
On the other hand, we observe that
\begin{align*}
 \Big\|\Loss(\algo_{\inputset},X_i,q) - \Loss(\algo_{\inputset^i},X'_i,q)\Big\| \leq \sqrt{Q} M.
\end{align*}
Indeed, the matrix $\Delta:=\Loss(\algo_{\inputset},X_i,q) -
\Loss(\algo_{\inputset^i},X'_i,q)$ is a matrix that is zero except for (possibly)
its $q$th row, that we may call $\bfdelta_q$. Thus:
$$\|\Delta\|=\sup_{\bfv:\|\bfv\|_2\leq 1}\|\Delta\bfv\|_2=\sup_{\bfv:\|\bfv\|_2\leq 1}\|\bfdelta_q\cdot\bfv\|=\|\bfdelta_q\|_2,$$
where $\bfv$ is a vector of dimension $Q$. Since each of the
$Q$ elements of $\bfdelta_q$ is in the range $[-M;M]$, we get that
$\|\bfdelta_q\|_2\leq \sqrt{Q}M.$

This allows us to conclude that
\begin{align}
 \|\Lqemp-\Lqiemp\| \leq \frac{2B}{m_q}+ \frac{\sqrt{Q} M}{m_q}\label{eq:lqemplqiemp}
\end{align}

Combining~\eqref{eq:lqlqi} and~\eqref{eq:lqemplqiemp} we just
proved that, for all $i$ such that $y_i=q$
$$(H_q(\bfZ_q)-H_q(\bfZ_q^i))^2\preccurlyeq
\left(\frac{4B}{m_q}+\frac{\sqrt{Q}M}{m_q}\right)^2\identity.$$
\qed
\end{proof}

We then establish the following Lemma
\begin{lemma}$\forall q,$
\begin{align*}
 \proba_{\bfX|\bfy}&\left\{\|\Lq-\Lqemp\|\geq t+\|\expectation_{\bfX|\bfy}[\Lq-\Lqemp]\|\right\}\leq 2Q \exp\left\{-\frac{t^2}{8 \left(\frac{4B}{\sqrt{m_q}}+
       \frac{\sqrt{Q} M}{\sqrt{m_q}}\right)^2}\right\}.
\end{align*}

\end{lemma}
\begin{proof}
Given the previous Lemma, Theorem~\ref{McD}, when applied on $H_q(\bfX_q,y_q)=\dilation(\Lq-\Lqemp)$ gives
$$\sigma^2_q =\left(\frac{4B}{m_q} + \frac{\sqrt{Q}
    M}{\sqrt{m_q}}\right)^2$$ to give, for $t>0$:
\begin{align*}
 \proba_{\bfX|\bfy}\left\{\|\Lq-\Lqemp-\expectation[\Lq-\Lqemp]\|\geq t\right\}\leq 2Q \exp\left\{-\frac{t^2}{8 \left(\frac{4B}{m_q}+ \frac{\sqrt{Q} M}{\sqrt{m_q}}\right)^2}\right\},
\end{align*}
which, using the triangle inequality $$|\|A\|-\|B\||\leq\|A-B\|,$$
gives the result. \qed
\end{proof}

Finally, we observe
\begin{lemma}$\forall q$,
\begin{align*}
 \proba_{\bfX|\bfy}&\left\{\|\Lq-\Lqemp\|\geq t+\frac{2B}{m_q}\right\}\leq 2Q \exp\left\{-\frac{t^2}{8 \left(\frac{4B}{ \sqrt{m_q}} + \frac{\sqrt{Q} M}{\sqrt{m_q}}\right)^2}\right\}.
\end{align*}
\end{lemma}
\begin{proof}
It suffices to show that 
\begin{align*}
\left\|\expectation[\Lq-\Lqemp]\right\|\leq \frac{2B}{m_q},
\end{align*}
and to make use of the previous Lemma.
We note
that for any $i$ such that $y_i=q,$ and for $X_i'$ distributed
according to $\distribution_{X|q}$:
\begin{align*}
\expectation_{\bfX|\bfy}\Lqemp&=\expectation_{\bfX|\bfy}\Loss_q(\algo_{\inputset},\bfX,\bfy)=\frac{1}{m_q}\sum_{j:y_j=q}\expectation_{\bfX|\bfy}\Loss(\algo_{\inputset},X_j,q)\\
&=\frac{1}{m_q}\sum_{j:y_j=q}\expectation_{\bfX,X_i'|\bfy}\Loss(\algo_{\inputset^i},X_i',q)=\expectation_{\bfX,X_i'|\bfy}\Loss(\algo_{\inputset^i},X_i',q).
\end{align*}
Hence, using the stability argument,
\begin{align*}
\|\expectation[\Lq-\Lqemp]\|
& =\left\|\expectation_{\bfX,X_i'|\bfy}\left[\Loss(\algo_{\inputset},X'_i,q)-\Loss(\algo_{\inputset^i},X_i',q)\right]\right\|\\
& \leq
\expectation_{\bfX,X_i'|\bfy}\left\|\Loss(\algo_{\inputset},X'_i,q)-\Loss(\algo_{\inputset^i},X_i',q)\right\|\\
& \leq
\expectation_{\bfX,X_i'|\bfy}\left\|\Loss(\algo_{\inputset},X'_i,q)-\Loss(\algo_{\inputset^{\backslash i}},X_i',q)\right\|\\
&\quad+\expectation_{\bfX,X_i'|\bfy}\left\|\Loss(\algo_{\inputset^i},X'_i,q)-\Loss(\algo_{\inputset^{\backslash
      i}},X_i',q)\right\|\\
&\leq \frac{2B}{m_q}.
\end{align*}

This inequality in combination with the previous lemma provides the result.\qed
\end{proof}

We are now set to make use of a union bound argument:
\begin{align*}
 \proba&\left\{\exists q:\|\Lq - \Lqemp\| \geq t + \frac{2B}{m_q}\right\}
 \leq
\sum_{q\in\targetspace}  \proba\left\{\exists q:\|\Lq - \Lqemp\| \geq t + \frac{2B}{m_q} \right\}\\
&\leq 2Q \sum_{q} \exp\left\{-\frac{t^2}{8 \left(\frac{4B}{\sqrt{m_q}} + \frac{\sqrt{Q} M}{\sqrt{m_q}}\right)^2}\right\}\leq 2Q^2 \max_q\exp\left\{-\frac{t^2}{8 \left(\frac{4B}{\sqrt{m_q}}+ \frac{\sqrt{Q} M}{\sqrt{m_q}}\right)^2}\right\}
\end{align*}
According to our definition $m^*$, we get
\begin{align*}
\proba&\left\{\exists q:\|\Lq - \Lqemp\| \geq t + \frac{2B}{m_q}\right\}
\leq  2Q^2 \exp\left\{-\frac{t^2}{8 \left(\frac{4B}{\sqrt{\mmin}} + \frac{\sqrt{Q} M}{\sqrt{\mmin}}\right)^2}\right\}.
\end{align*}

Setting the right hand side to $\delta$, gives the result of Theorem~\ref{th:confusionbound}.


%% file: analysis.tex
\section{Analysis of existing algorithms}
\label{sec:analysis}
Now that the main result on stability bound has been established, we will investigate how existing multiclass algorithms
exhibit stability properties and thus fall in the scope of our analysis.
More precisely, we will analyse two well-known models for multiclass
support vector machines and we will show that they may promote small
confusion error.
But first, we will study the more general stability of multiclass algorithms using regularization in Reproducing Kernel Hilbert Spaces (RKHS).

\subsection{Hilbert Space Regularized Algorithms}
Many well-known and widely-used algorithms feature a minimization of a
regularized objective functions \cite{tikhonov77}.
In the context of multiclass kernel machines \cite{crammer01algorithmic,cristianini00}, this regularizer $\Omega(h)$ may take the following form:
\begin{align*}
 \Omega(h) = \sum_{q} \|h_q\|_k^2.
\end{align*}
where $k:\inputspace \times \inputspace \to \realset$ denotes the kernel associated to the RKHS $\rkhs$.

In order to study the stability properties of algorithms, minimizing a data-fitting term, penalized by such regularizers,
 in our multi-class setting, we need to
introduce a minor definition that is an addition to definition 19 of \cite{bousquet02}.
\begin{definition}
\label{defadm}
 A loss function $\ell$ defined on $\rkhs^Q \times \targetspace$ is $\sigma$-multi-admissible if $\ell$ is $\sigma$-admissible
with respect to any of his $Q$ first arguments.
\end{definition}

This allows us to come up with the following theorem.
\begin{theorem}
\label{rkhsstab}
Let $\rkhs$ be a reproducing kernel Hilbert space (with kernel $k$) such that $\forall X \in \inputspace, k(X,X) \leq \kappa^2 < +\infty$.
Let $L$ be a loss matrix, such that $\forall q \in \targetspace$, $\ell_q$ is $\sigma_q$-multi-admissible.
And let $\algo$ be an algorithm such that
\begin{align*}
\label{rkhsalgo}
 \algo_{\setS} &= \argmin_{h \in \rkhs^Q} \sum_{q} \sum_{n:y_n=q} \frac{1}{m_q} \ell_q(h,x_n,q) + \lambda \sum_{q} \|h_q\|_k^2.\\
:&= \argmin_{h \in \rkhs^Q} J(h).
\end{align*}

 Then $\algo$ is confusion stable with respect to the set of loss
 functions $\bfloss$.
Moreover, a $B$ value defining the stability is
\begin{align*}
 B= \max_q\frac{\sigma_q^2 Q \kappa^2}{2 \lambda},
\end{align*}
where $\kappa$ is such that $k(X,X) \leq \kappa^2 < +\infty$
\end{theorem}

\begin{proof}[Sketch of proof]
In essence the idea is to exploit Definition \ref{defadm} in order to apply Theorem 22 of \cite{bousquet02}
for each loss $\ell_q$.
Moreover our regularizer is a sum (over $q$) of RKHS norms, hence the
additional $Q$ in the value of $B$.\qed
\end{proof}

\subsection{Lee, Lin and Wahba model}
One of the most well-known and well-studied model for multi-class classification, in the context of SVM,
was proposed by \cite{lee04}.
In this work, the authors suggest the use of the following loss function.
\begin{align*}
 \ell(h,x,y) &= \sum_{q \neq y} \left( h_q(x) + \frac{1}{Q-1} \right)_+
\end{align*}
Their algorithm, denoted $\algo^{\text{LLW}}$, then consists in minimizing the following (penalized) functional,
\begin{align*}
 J(h) = \frac{1}{m} \sum_{k=1}^m \sum_{q \neq y_k} \left( h_q(x_k) + \frac{1}{Q-1} \right)_+ + \lambda \sum_{q=1}^Q \|h_q\|^2,
\end{align*}
with the constraint $\sum_q h_q = 0$.

We can trivially rewrite $J(h)$ as
\begin{align*}
  J(h) = \sum_{q} \sum_{n:y_n=q} \frac{1}{m_q} \ell_q(h,x_n,q) + \lambda \sum_{q=1}^Q \|h_q\|^2,
\end{align*}
with
\begin{align*}
 \ell_q(h,x_n,q) = \sum_{p\neq q} \left(h_p(x_k) + \frac{1}{Q-1}\right)_+.
\end{align*}

It is straightforward that for any $q$, $\ell_q$ is $1$-multi-admissible.
We thus can apply theorem \ref{rkhsstab} and get $B={Q \kappa^2}/{2 \lambda}$.

\begin{lemma}
\label{lemmaM}
 Let $h^*$ denote the solution found by $\algo^{\text{LLW}}$.
$\forall x \in \inputspace, \forall y \in \targetspace, \forall q$, we have $$\ell_q(h^*,x,y) \leq \frac{Q \kappa}{\sqrt{\lambda}} +1.$$
\end{lemma}
\begin{proof}
 As $h^*$ is a minimizer of $J$, we have
\begin{align*}
 J(h^*) \leq J(0) = \sum_{q} \sum_{n:y_n=q} \frac{1}{m_q} \ell_q(0,x_n,q) = \sum_{q} \sum_{n:y_n=q} \frac{1}{(Q-1) m_q} = 1.
\end{align*}
As the data fitting term is non-negative, we also have
\begin{align*}
 J(h^*) \geq \lambda \sum_{q} \|h_q^*\|_k^2.
\end{align*}
Given that $h^* \in \rkhs$, Cauchy-Schwarz inequality gives
\begin{align*}
 \forall x \in \inputspace, \|h_q^*\|_k \geq \frac{|h_q^*(x)|}{\kappa}.
\end{align*}
Collecting things, we have
\begin{align*}
 \forall x \in \inputspace, |h_q^*(x)| \leq \frac{\kappa}{\sqrt{\lambda}}.
\end{align*}
Going back to the definition of $\loss_q$, we get the result.\qed
\end{proof}

Using theorem \ref{th:confusionbound}, it follows that, with probability $1 - \delta$,
\begin{align*}
\left\|\widehat{\confusion}_{\bfY}(\algo^{\text{LLW}},\bfX)-\confusion_{\bfs(\bfY)}(\algo^{\text{LLW}})\right\|\leq
\sum_{q} \frac{Q \kappa^2}{\lambda m_q}+ \frac{\sqrt{8 \ln \left(\frac{Q^2}{\delta}\right)}\left(\frac{2 Q^2 \kappa^2}{\lambda} + \left(\frac{Q \kappa}{\sqrt{\lambda}} +1\right) Q \sqrt{Q}\right)}{\sqrt{\mmin}}.
\end{align*}

\subsection{Weston and Watkins model}
Another multiclass mode is due to \cite{weston98}.
They consider the following loss functions.
\begin{align*}
 \ell(h,x,y) &= \sum_{q \neq y} \left(1 - h_y(x) + h_q(x)\right)_+\\
\end{align*}
The algorithm $\algo^{\text{WW}}$ minimizes the following functional
\begin{align*}
 J(h) = &\frac{1}{m} \sum_{k=1}^m \sum_{q \neq y_k} \left(1 - h_y(x) + h_q(x)\right)_++ \lambda \sum_{q<p=1}^Q \|h_q - h_p\|^2,
\end{align*}

This time, for $1\leq p,q\leq Q$, we will introduce the functions $h_{pq} = h_p - h_q$.
We can then rewrite $J(h)$ as
\begin{align*}
  J(h) = \sum_{q} \sum_{n:y_n=q} \frac{1}{m_q} \ell_q(h,x_n,q) + \lambda \sum_{p=1}^Q \sum_{q=1}^{p-1} \|h_{pq}\|^2,
\end{align*}
with
\begin{align*}
 \ell_q(h,x_n,q) = \sum_{p\neq q} \left(1 - h_{pq}(x_n)\right)_+.
\end{align*}

It still is straightforward that for any $q$, $\ell_q$ is $1$-multi-admissible.
However, this time, our regularizer consists in the sum of $\frac{Q (Q-1)}{2} < \frac{Q^2}{2}$ norms.
Applying Theorem \ref{rkhsstab} therefore gives $B=  {Q^2 \kappa^2}/{4 \lambda}.$

\begin{lemma}
 Let $h^*$ denote the solution found by $\algo^{\text{WW}}$.
$\forall x \in \inputspace, \forall y \in \targetspace, \forall q$, we have $\ell_q(h^*,x,y) \leq Q \left(1+ \kappa \sqrt{\frac{Q}{\lambda}}\right)$. 
\end{lemma}
This lemma can be proven following exactly the same techniques and reasoning as Lemma \ref{lemmaM}.

Using theorem \ref{th:confusionbound}, it follows that, with probability $1 - \delta$,
\begin{align*}
&\left\|\widehat{\confusion}_{\bfY}(\algo^{\text{WW}},\bfX)-\confusion_{\bfs(\bfY)}(\algo^{\text{WW}})\right\| \leq
\sum_{q} \frac{Q^2 \kappa^2}{2 \lambda m_q}+ \frac{\sqrt{8 \ln \left(\frac{Q^2}{\delta}\right)}\left(\frac{Q^3 \kappa^2}{\lambda} + Q^2 \Big(\sqrt{Q} + \kappa \frac{Q}{\sqrt{\lambda}}\Big)\right)}{\sqrt{\mmin}}.
\end{align*}


%% file: conclusion.tex
\section{Discussion and Conclusion}
\label{sec:conclusion}

In this paper, we have proposed a new framework, namely the
algorithmic {\em confusion
stability}, together with new
bounds to characterize the generalization properties of
 multiclass learning algorithms. The crux of our study is to envision the
confusion matrix as a performance measure, which differs from commonly
encountered  approaches that  investigate generalization properties of
scalar-valued performances.  

A few questions that are raised by the present work are the
following. Is it possible to derive confusion stable algorithms
that precisely aim at controlling the norm of their confusion matrix?
Are there other algorithms than those analyzed here that may be
studied in our new framework? On a broader perspective: how can noncommutative
concentration inequalities be of help to analyze complex settings
encountered in machine learning (such as, e.g., structured prediction,
operator learning)?
